\documentclass[12pt]{amsart}   

\usepackage{graphicx}
\usepackage{tikz}
\usepackage{amssymb}
\usetikzlibrary{calc}
\usepackage{url}
\usepackage{caption}
\usepackage{tcolorbox}
\usepackage{listings}

\newtheorem{thm}{Theorem}[section]

\theoremstyle{definition}

\theoremstyle{remark}

\begin{document}

\title{Automated Conjecturing VII: \\The Graph Brain Project\\  \& Big Mathematics
}

\author{N. Bushaw, C. E. Larson$^*$, N. Van Cleemput$^{1}$}

\author{ \\with Summer 2017 workshop participants:\\ R. Barden, C. Callison, A. Fernandez$^2$, B. Harris, I. Holden,\\ D. Muncy, C. O'Shea$^3$, J. Shive, J. Raines, P. Rana, B. Ward,\\ N. Wilcox-Cook}


\address{Department of Mathematics and Applied Mathematics\\Virginia Commonwealth University\\Richmond, VA 23284, USA }
\address{(1) Department of Applied Mathematics, Computer Science and Statistics\\Ghent University\\  9000 Ghent, Belgium}
\address{(2) Trinity Episcopal School, Richmond, VA 23235.}
\address{(3) Mills E. Godwin High School, Richmond, VA 23238.}
\thanks{(*) Research supported by the Simons Foundation Mathematics and Physical Sciences--Collaboration Grants for Mathematicians Award (426267), and Virginia Commonwealth University--Presidential Research Inception Program (PRIP)}


\date{}

\maketitle

%
%
%

\section{Introduction}

Our Project began with the development of a program that can be used to generate invariant-relation and property-relation  conjectures in many areas of mathematics. This program can produce conjectures which are not implied by existing (published) theorems. 
Here we propose a new approach to push forward existing mathematical research goals---using automated mathematical discovery software. We suggest how to initiate and harness  large-scale collaborative mathematics. We envision mathematical research labs similar to what exist in other sciences, new avenues for funding, new opportunities for training students, and a more efficient  and effective use of published mathematical research.

The Graph Brain Project is an experiment in how the use of automated mathematical discovery software, databases, large collaboration, and systematic investigation provide a model for how mathematical research might proceed in the future. Our experiment is \textit{modular} and can be usefully expanded. We investigated one small open problem in graph theory. In the course of this investigation we coded many graph theoretic concepts and graphs, and computed values of many invariants for these graphs. Other researchers working on other open problems, adding their own contributions and expertise, and following their own graph theoretic interests, can leverage and supplement our code---a multiplier effect.
And our experiment in graph theory can be imitated in many other areas of mathematical research. Big Mathematics is the idea of large, systematic, collaborative research on problems of existing mathematical interest. What is possible when we put our skills, tools, and results together systematically?

Automated mathematical discovery programs are at the point where their utility  to researchers cannot be ignored.
Conjectures are the life-blood of mathematics. 
The papers \cite{LarsVanc16,HutcEtal17} include examples of automated conjectures for matrix theory, number theory, graph theory and chemical graph theory; these are of the form of bounds for matrix, integer and graph invariants. In other research we have generated conjectures for combinatorial games, intersecting set systems, and linear programs, among others: the idea is general---all you need to get going are a few coded invariants and example objects.  
That said, as we are able to coax our machines to do more and more things that historically required human ingenuity, human mathematicians will always have an essential role: computer contributions are necessarily judged by how much they help us achieve our human mathematical goals; our human questions are our yardsticks with which we measure our computer assistants.

\begin{figure}

\includegraphics[width=3.95in]{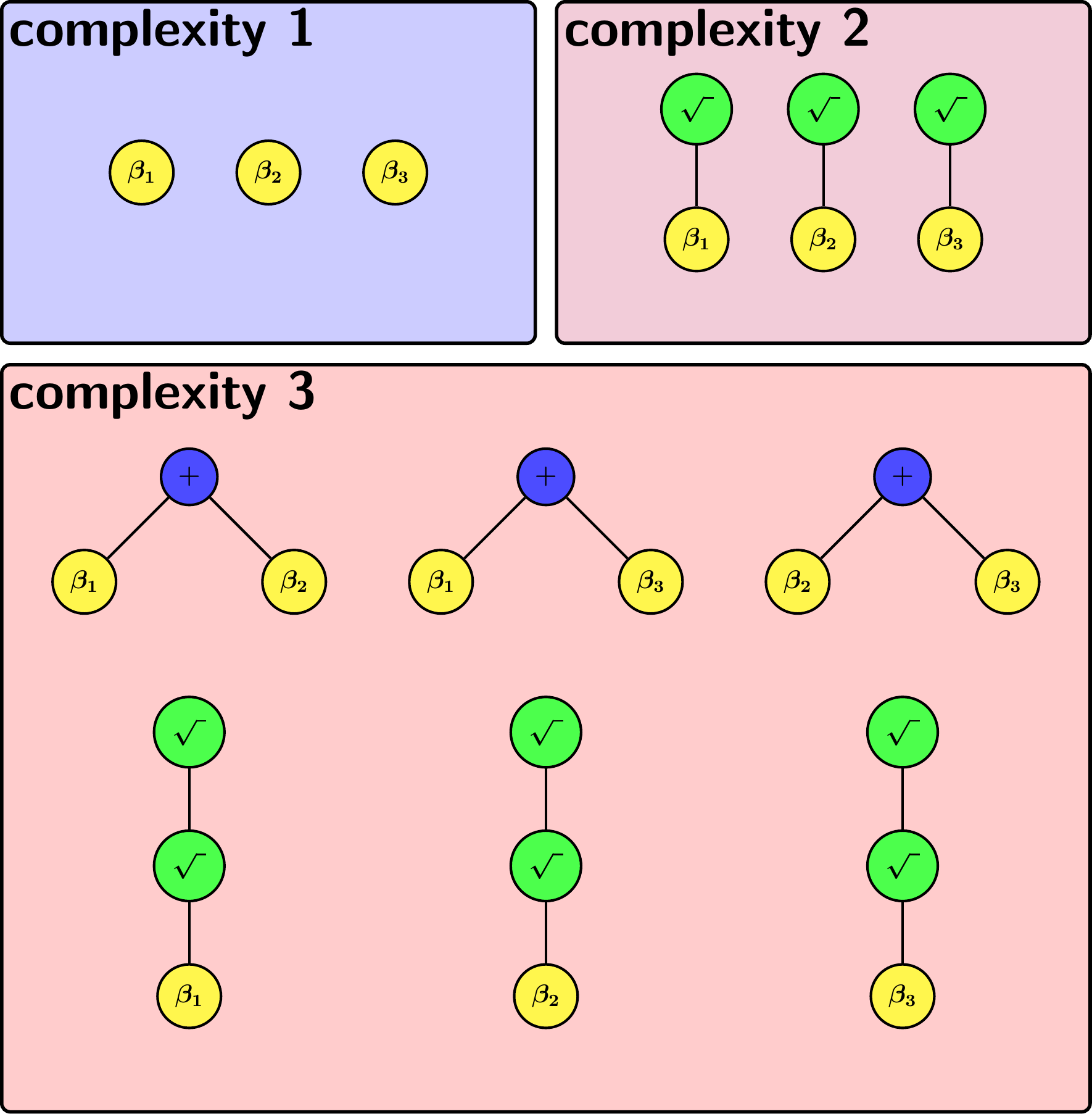}

\caption{These are all of the \textit{expressions} of complexity no more than $3$ that can be formed from invariants 
$\beta_1$, $\beta_2$, and $\beta_3$ and operators $+$ and $\sqrt{}$. 
These are real-valued functions that can be applied to objects of type corresponding to the invariants.}
\end{figure}

A central idea is that computers can easily exhaustively generate and evaluate all expressions formed from standard mathematical ingredients for relatively small numbers of example mathematical objects. 
These expressions---and their corresponding values for example sets of objects---can then be utilized in a variety of ways.
In well-defined instances it can be argued that no human can find a simpler expression satisfying certain conditions using the same mathematical ingredients. Consider the search for upper bounds for some invariant $\alpha$. For these purposes we take an \textit{invariant} to be a real-valued function of the objects. To be concrete, assume the objects are (finite) \textit{graphs} and that $\beta_1$, $\beta_2$, \ldots, $\beta_k$ are graph invariants. 
An upper bound $\beta$ of $\alpha$ will be some mathematical function of the $\beta_i$'s. This function may involve arithmetic operations, algebraic operations, or any other mathematical operations. Some real-number operators would include addition and square root ($+$ and $\sqrt{ }$). 
In this case (and if $k=3$ for the sake of example) the complexity-1 expressions would be the invariants themselves: $\beta_1,\beta_2,\beta_3$. The complexity-2 expressions would be all the expressions that can be formed from a mix of two operators or invariants. Since $\sqrt{ }$ is the only unary operator, the only possibilities are: $\sqrt{\beta_1}$, $\sqrt{\beta_2}$, $\sqrt{\beta_3}$. The complexity-3 expressions are:  
$\beta_1+ \beta_2$, $\beta_1 + \beta_3$, $\beta_2+ \beta_3$, $\beta_1 + \beta_1$, 
$\beta_2+\beta_2$, $\beta_3+\beta_3$, 
 $\sqrt{\sqrt{\beta_1}}$, $\sqrt{\sqrt{\beta_2}}$, and $\sqrt{\sqrt{\beta_3}}$ (a modern computer algebra system can identify and remove expressions equivalent due to additive commutativity, etc). A program can recursively generate all possible $\beta$'s 
up to any specified complexity. Generating expressions will face combinatorial explosion---but there is no difficulty in generating all (relatively small) human-comprehendible expressions. (Our program can generate more than 100 million expressions per second, depending on the complexity of the expressions, on a standard laptop). 
Our \textsc{conjecturing} program can either evaluate these expressions on the fly for a particular object (graph) or, better, access a database of pre-computed invariant values. These generated, evaluated expressions---together with a list of existing bounds for the invariant $\alpha$---are the main ingredients in generating conjectures that improve on all published bounds for $\alpha$. 

\begin{figure}
\begin{tcolorbox}
{\small
\begin{tabular}{c|c|c}

Objects & Invariants & Properties\\
\hline\hline
Graphs & \texttt{independence\_number},  & \texttt{is\_tree}, \\
& \texttt{radius} & \texttt{is\_hamiltonian}\\
\hline
Symmetric  & \texttt{det}, & \texttt{is\_unitary}, \\
Matrices & \texttt{max\_eigenvalue}  & \texttt{is\_positive\_definite}\\
\hline
Natural & \texttt{distinct\_prime\_factors}, & \texttt{is\_prime}, \\
Numbers  & \texttt{largest\_prime\_factor} & \texttt{is\_even}

\end{tabular}
}
\end{tcolorbox}
\caption{A user of  \textsc{Conjecturing} will need to input some examples of objects, invariants and properties.}
\end{figure}

In the case of bounds for an invariant of a mathematical object, the program functions best the more relevant invariants that it has available for its produced conjectures. That is, if there are unknown bounds that are functions of invariants known to mathematicians (and recorded in the mathematics literature) these will be produced by the program if the invariants and properties are included in the program. In particular, \textit{the program produces the simplest (in terms of complexity) bounds that are true with respect to the objects that it knows using the invariants  that it knows, and any other given constraints}.

The foundational idea of Big Mathematics is to form research groups of various sizes to work on specific mathematical problems using automated discovery tools, databases, and exploiting the mathematical literature systematically. Some members of a research group might code invariants, objects and properties. Other members can be in charge of generating conjectures (which can be done automatedly), and testing conjectures and finding counterexamples (which can be done systematically for small objects if object-generators are coded). Other members can work on proving conjectures. A group might have a library specialist (responsible for identifying all existing theorems that are relevant for an investigation, and keeping track of new concepts to be coded), a code-management and database specialist (to maintain stable code, manage versioning and code updates, and compute and store values for all coded invariants for all coded objects).
In order to maximize what is possible research groups will need to code huge bodies of published mathematical research. This research, in any mathematical sub-field, consists of large numbers of published examples of mathematical objects, invariants and properties of those objects, together with other related concepts. A nice feature is that, once coded, any other researcher or group can use and build on this work. Ideally we could build code-bases of graph theoretic knowledge that make it easy and profitable to use and extend---and enjoy network effects.


In the following sections we mention the historical context of our research---which goes back to the earliest days of computer science and artificial intelligence research. We discuss an example that demonstrates what is currently possible. And finally we discuss how this example---and our Graph Brain Project---can be ramped up to help mathematicians more quickly---and systematically---attain our shared research goals.

\section{Background}

This Graph Brain Project is motivated by our research in automated mathematical conjecturing programs---a small part of the larger area of automated mathematical discovery research. Alan Turing, in a 1948 report on ``Intelligent Machinery'', suggested mathematics as a domain to begin with in building a ``thinking machine'' \cite{Turi48}. There has ever since been some number of researchers working to automate parts of mathematics, with varying success, and in developing computer tools that provide intelligent assistance to mathematics researchers.

\begin{figure}
\begin{tabular}{ccc}
\includegraphics[width=1.2in]{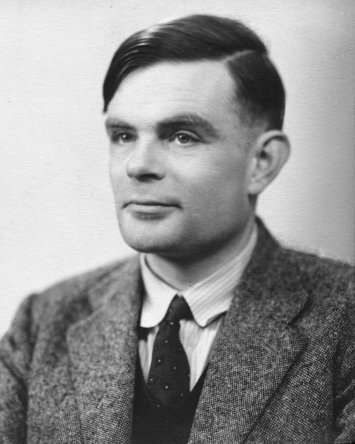} &
\includegraphics[width=1.2in]{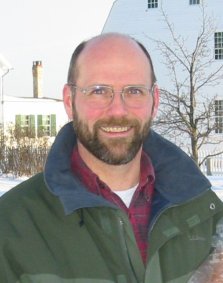} &
\includegraphics[width=2.5in]{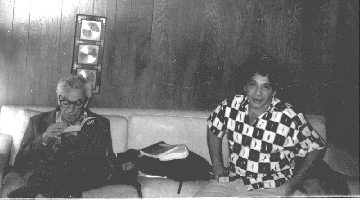}
\end{tabular}
\caption{Alan Turing; William McCune; Paul Erd\H{o}s \& Siemion  Fajtlowicz. Erd\H{o}s was as well-known for his conjectures as for his theorems}.
\end{figure}

Automated theorem proving was the first and has been the most studied area.  The first programs to prove theorems were developed in the 1950s \cite{NeweSimo58}; and the McCune/Otter 1996 computer proof of the Robbins Conjecture \cite{Mccu97} was a milestone in this area. Zeilberger has done impressive research on the automatic proof of conjectured combinatorial identities \cite{WilfZeil90}. 
The first program to make conjectures leading to published mathematical research was Fajtlowicz's Graffiti program \cite{Fajt88a}. 
Research on integer relation detection between sets of numbers has led to surprising conjectures and breakthroughs, including a new formula for the digits of the number $\pi$ \cite{BorwBail04}. 
Of course, all this is only a small part of what mathematical research consists in---or of what might be attempted. 

\begin{figure}
\begin{tabular}{ccc}
\includegraphics[width=1.5in]{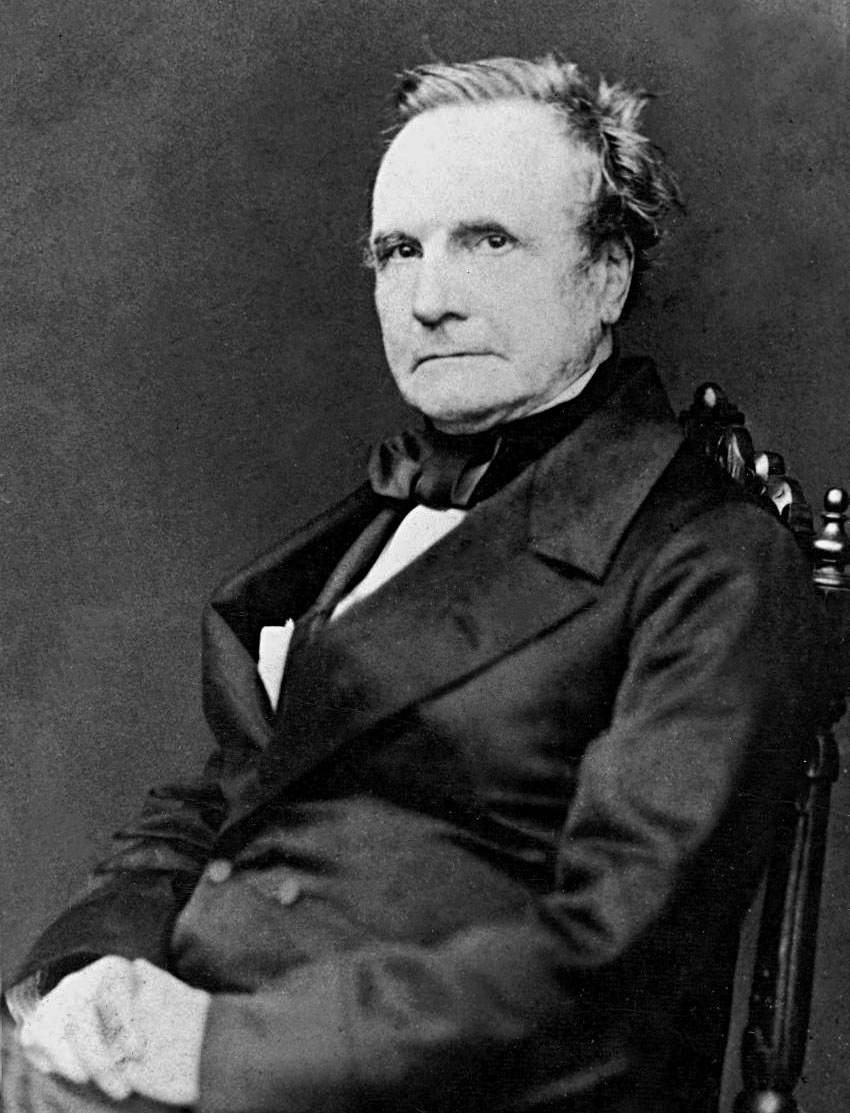} &
\includegraphics[width=1.6in]{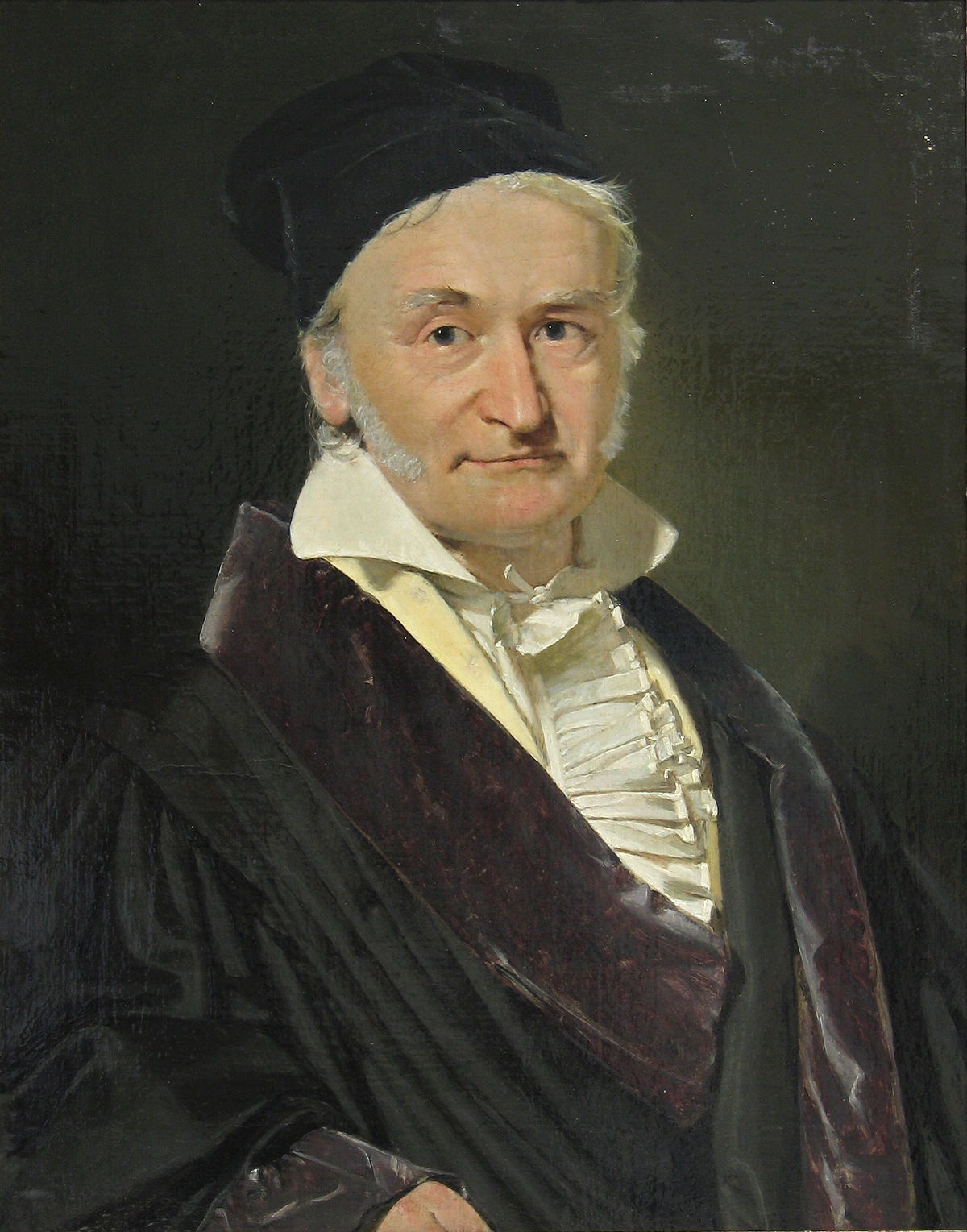} &
\includegraphics[width=1.4in]{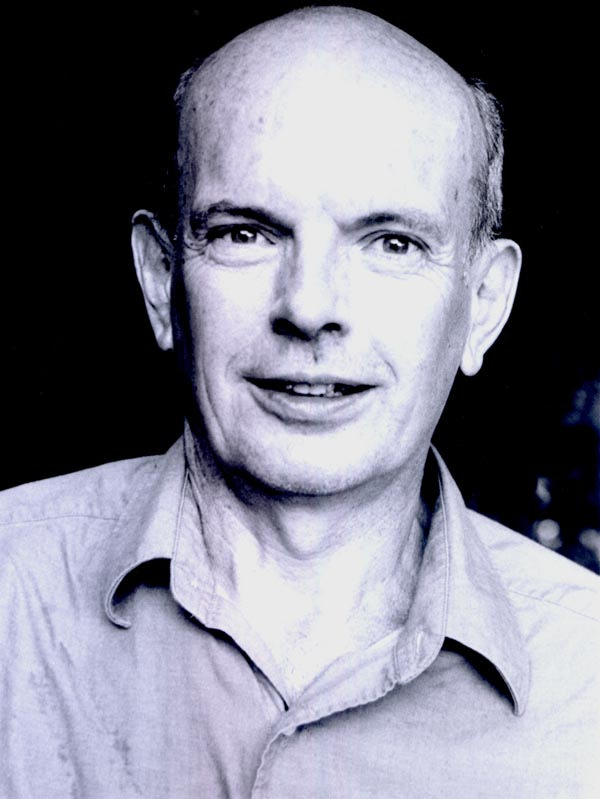}
\end{tabular}
\caption{Charles Babbage, Carl Friedrich Gauss, and Neil Sloane}
\end{figure}

We demonstrate that building and maintaining databases of non-trivial computational results---for instance, values of NP-complete graph invariants for all published graphs---will be generally useful in scientific research; this can be coordinated and standardized, and will be a component of Big Mathematics. 
The utility of significant computations has a long history in our subject, going back at least to Ptolemy's trigonometry tables, and more recent log tables \cite{CambEtal03}. It should be noted that Babbage promoted his Difference Engine  to have the advantage of producing accurate mathematical tables free of human calculating error---and this may have been the largest funded mathematics-related project ever \cite{Essi14}. Accurate computations are not only valuable for engineering purposes, but even for purely mathematical investigations: Gauss conjectured the Prime Number Theorem on the basis of the table of primes he had computed. The Online Encyclopedia of Integer Sequences (OEIS, initiated by Neil Sloane 50 years ago \cite{Sloa07}), a 21$^{st}$ century analog of Gauss' tables, which makes essential use of modern computer resources, is a familiar tool for many researchers  searching for patterns. 

\begin{figure}
\begin{tabular}{cc}
\includegraphics[width=2.5in]{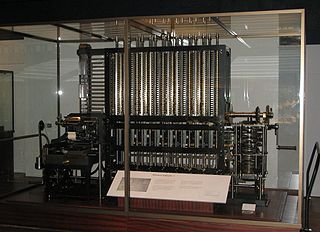} &
\includegraphics[width=2.6in]{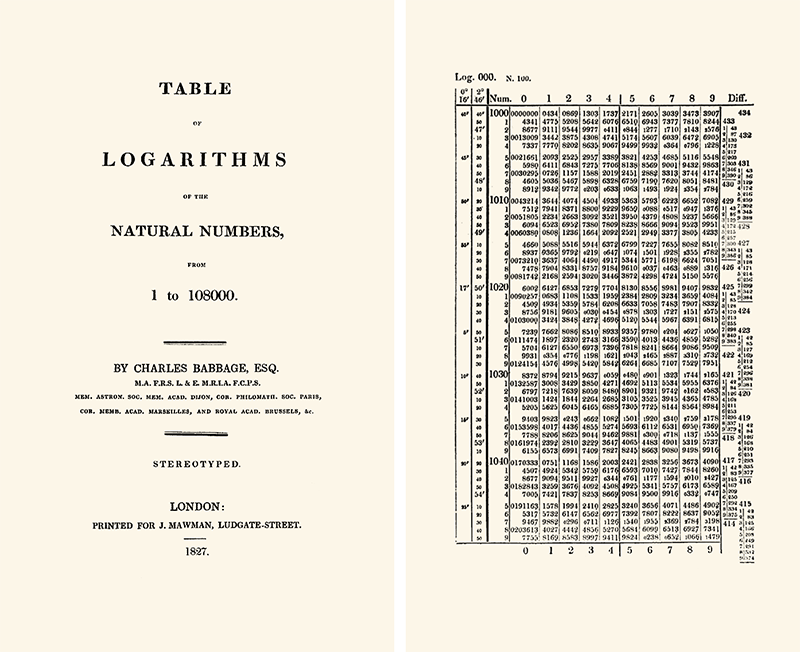}
\end{tabular}
\caption{Babbage's Difference Engine, and associated logarithm tables.}
\end{figure}

Larson and Van Cleemput have  developed a general-purpose conjecturing program---built around Fajtlowicz's Dalmatian heuristic---that has demonstrated its utility for a number of areas of mathematical research \cite{LarsVanc16}. 
Generated expressions function as conjectured bounds for an investigated invariant. These are tested for truth with respect to the stored objects. Conjectures are not stored or produced unless they imply a better approximate value  for at least one coded object than any coded theoretical bound or previously stored conjectured bound. 

\begin{figure}
\includegraphics[width=3.8in]{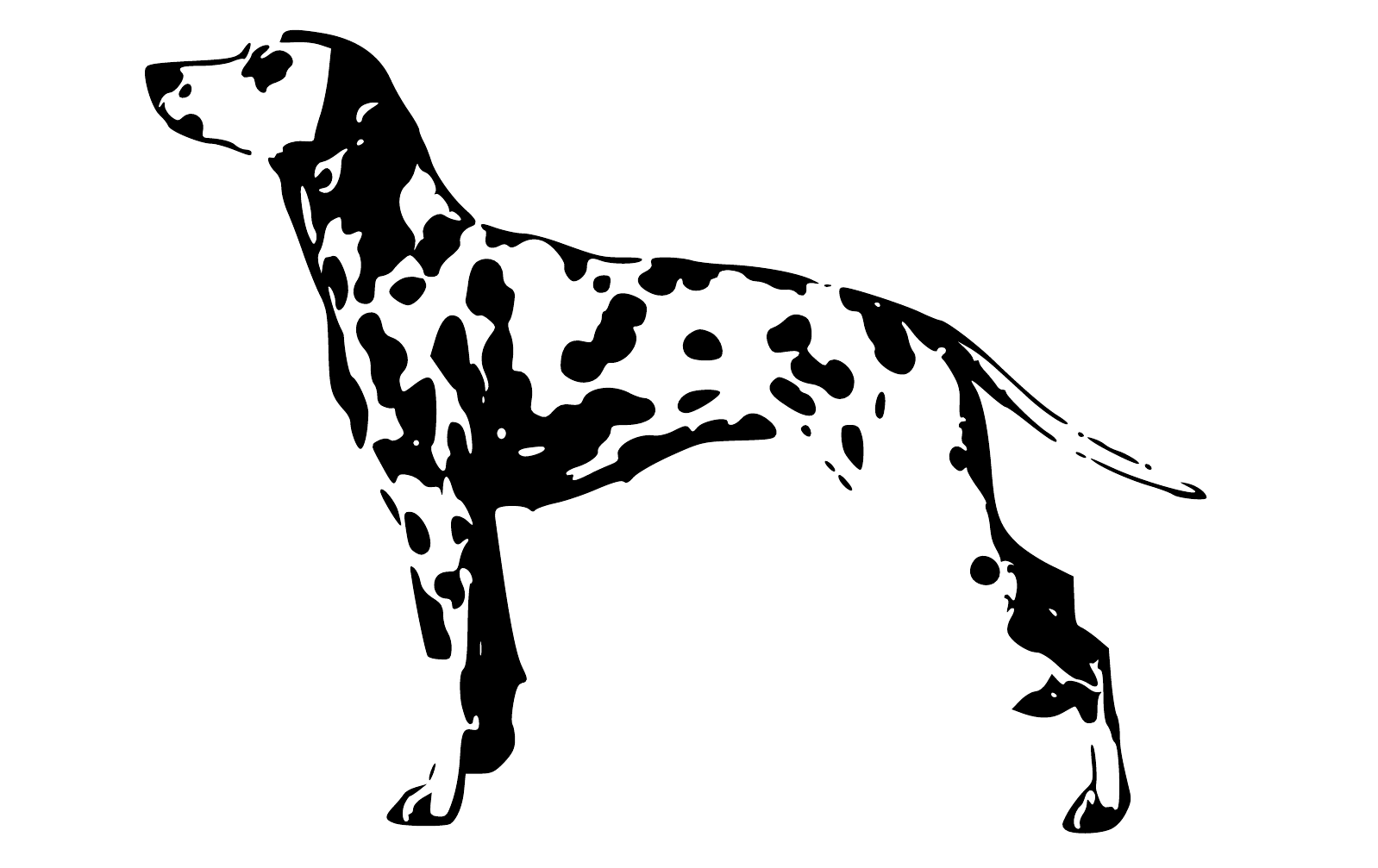} 
\begin{tikzpicture}[scale=.7]
\begin{scope}
\clip (0,6) parabola[bend at end] (4,4) parabola[bend at start] (9,3) --
      (9,1) parabola[bend at start] (4,2) parabola[bend at end] (0,3) -- (0,6) -- cycle;

\path
    coordinate (g1) at (1,5)
    coordinate (g2) at (2,4)
    coordinate (g3) at (3,4)
    coordinate (g4) at (4,2.1)
    coordinate (g5) at (5,3.9)
    coordinate (g6) at (6,3)
    coordinate (g7) at (7,2)
    coordinate (g8) at (8,1.1)
    ;

\filldraw [black] 
     (g1) circle (4pt) node[black] {}
     (g2) circle (4pt) node[black] {}
     (g3) circle (4pt) node[black] {}
     (g4) circle (4pt) node[black] {}
     (g5) circle (4pt) node[black] {}
     (g6) circle (4pt) node[black] {}
     (g7) circle (4pt) node[black] {}
     (g8) circle (4pt) node[black] {};
\end{scope}

\draw[->] (0,0) -- (9,0) node[anchor=north] {};
\draw	(0,0) node[anchor=north] {}
		(1,0) node[anchor=north] {$G_1$}
		(2,0) node[anchor=north] {$G_2$}
		(3,0) node[anchor=north] {$G_3$}
		(4,0) node[anchor=north] {$G_4$}
		(5,0) node[anchor=north] {$G_5$}
		(6,0) node[anchor=north] {$G_6$}
		(7,0) node[anchor=north] {$G_7$}
		(8,0) node[anchor=north] {$G_8$}
		(9,3) node[anchor=west] {$\min\{\alpha\texttt{ Upper Bounds}\}$}
		(9,1) node[anchor=west] {$\max\{\alpha\texttt{ Lower Bounds}\}$} 

		;
\draw[->] (0,0) -- (0,7) node[anchor=east] {$\alpha$};

\draw[dashed] (0,6) parabola[bend at end] (4,4) parabola[bend at start] (9,3);
\draw[dashed] (0,3) parabola[bend at start] (4,2) parabola[bend at end] (9,1);

\end{tikzpicture}
\caption{Fajtlowicz's Dalmatian heuristic. Graphs $G_i$ are on the horizontal axis. Conjectured bounds provide maximum and minimum values which can be used to estimate the independence number $\alpha$: the true values of $\alpha$ are spots between the curves of these theoretical ranges.}
\end{figure}

In some instances we have been able to prove the conjectures of our program---two new  \textit{theorems} are reported here. 
One attractive theorem resulted from our 2015 summer project investigating the combinatorial game Chomp \cite{BradEtal17}; several more resulted from our 2016 summer project investigating the domination number of benzenoids \cite{HutcEtal17}. 
\begin{figure}[h]
\begin{tabular}{cc}
\includegraphics[width=3in]{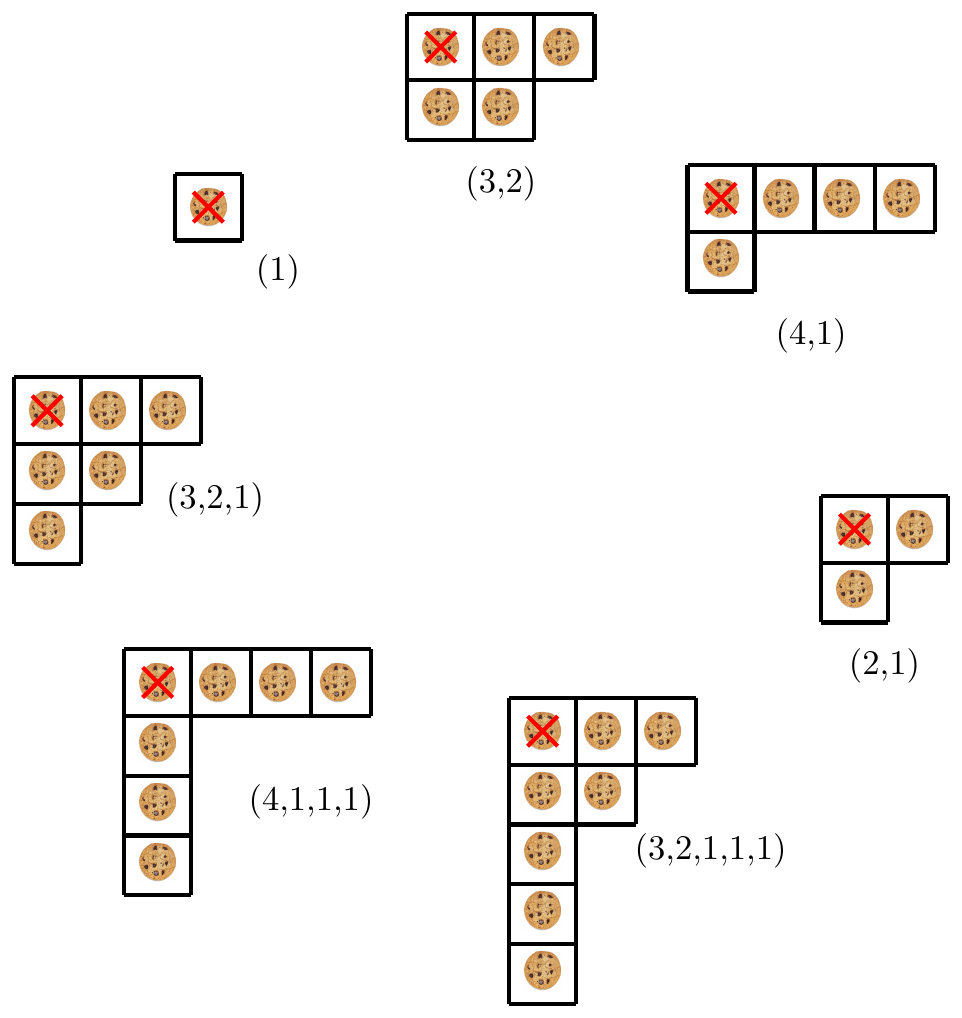} &
 \includegraphics[width=2in]{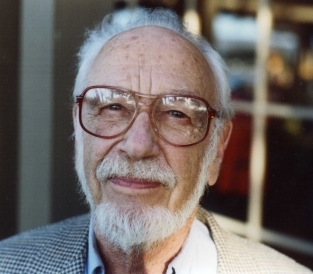}
\end{tabular}
\caption{David Gale, Chomp board positions. One conjecture led to the theorem that, for any position where the previous player to play has a winning strategy, the number of remaining cookies is at least one less than twice the number of non-empty columns.}
\end{figure}

We are graph theorists. The best approach to demonstrate the utility of the kind of research programs that we are advocating is to attempt this research for graphs. 
A \textit{graph} (or \textit{network}) is a mathematical object consisting of vertices and edges between them. Graphs are used to model many situations: these include molecular structure \cite{Dias87,CyviGutm89, FowlMano95}, the World Wide Web \cite{PageBrinMotw99}, social networks \cite{newman2002random}, and GPS satellite networks \cite{BrinCrevFrye11}.  And results in graph theory can be used as tools for proving results in other areas of mathematics: one very nice example is the proof of the Birkhoff-von Neumann theorem (that every doubly stochastic matrix can be written as a convex combination of permutation matrices) using the K\"{o}nig-Egervary theorem (that the covering number of a bipartite graph equals its  matching number) \cite{LovaPlum86}. 

\begin{figure}[h]
\begin{tabular}{cc}
\includegraphics[width=2.65in]{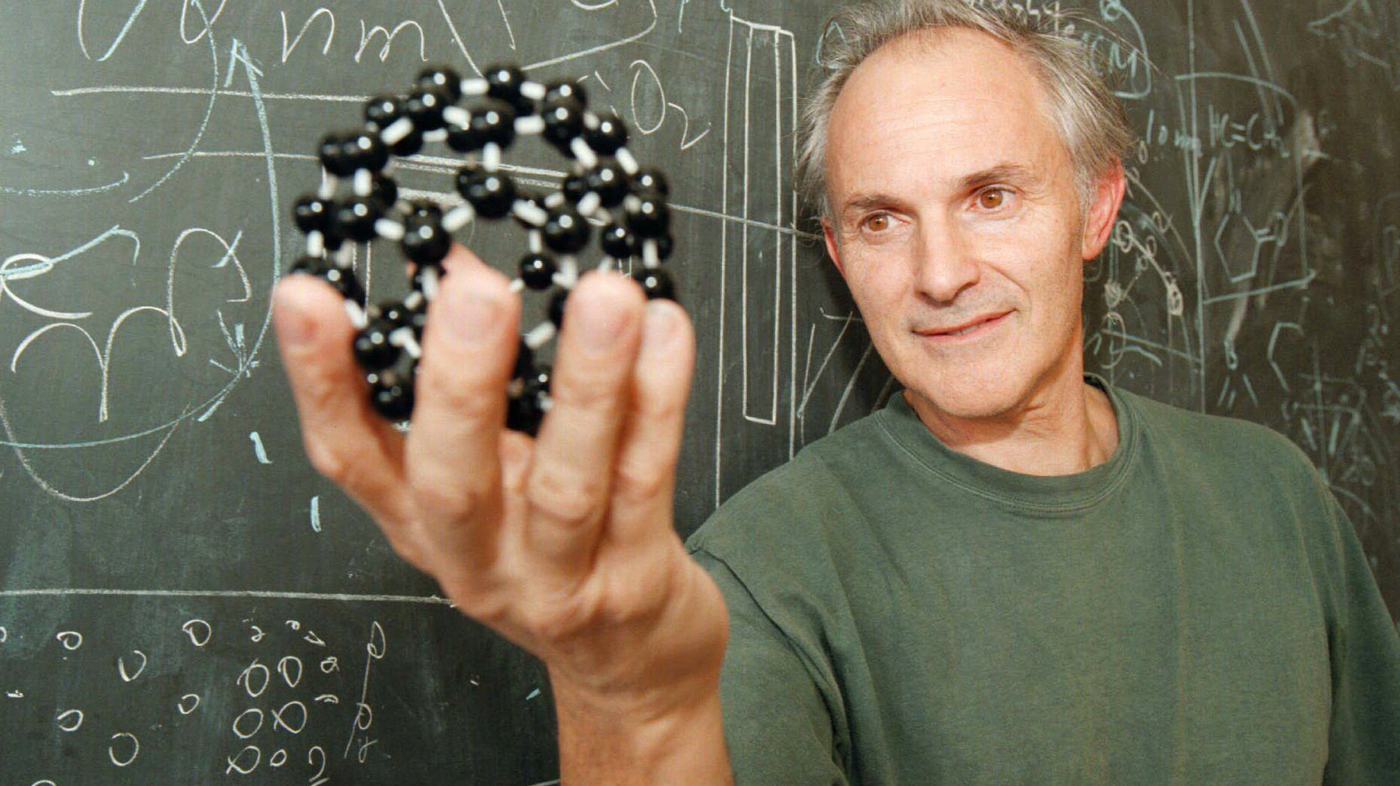} &
\includegraphics[width=1.75in]{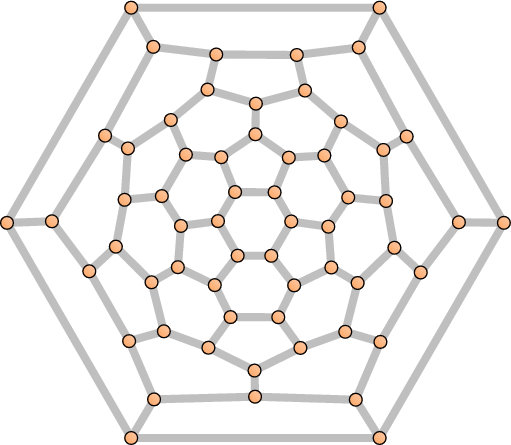}
\end{tabular}
\caption{Sir Harold Kroto, co-discoverer of fullerene molecules, holding a model of a buckyball; a graph of buckminsterfullerine $C_{60}$.}
\end{figure}

We will demonstrate the potential of our approach by investigating conjectured bounds for the independence number of a graph, a fundamental graph theory concept, intractable, and computationally equivalent to hundreds of other concepts in discrete mathematics. We have generated new conjectured bounds for the independence number of a graph which are not implied by any existing (published) bounds.


\section{Independence Number and the \textsc{conjecturing} Program}





The \textit{independence number} (or \textit{stability number}) $\alpha$ of a graph is the largest number of points in the graph where no pair of the points has a line between them. It is a widely studied hard-to-compute graph invariant which arises in a variety of situations.   Calculating the independence number of a graph can be used to optimize the configuration of a GPS network. Stable benzenoids \cite{Pepp08} and small stable fullerenes tend to minimize their independence numbers \cite{FajtLars03}. The independence number of a graph is a central concept of two of the most studied and important problems in graph theory: the P~vs.~NP question \cite{GareJohn79}, and Hadwiger's Conjecture \cite{DuchMeyn82,MaffMeyn87,Chud10}. Many families of combinatorial objects including error-correcting codes, set packings in Hamming spaces, and balanced incomplete block designs can be viewed as maximum independent sets \cite{Oste05}.

\begin{figure}[h]
\begin{tabular}{ccc}
\includegraphics[scale=.5]{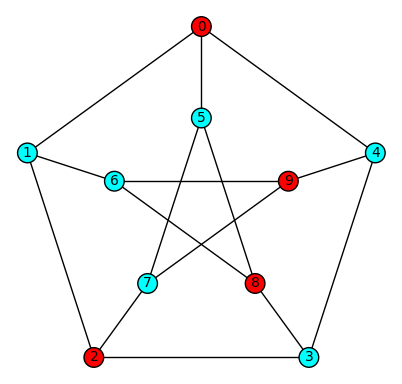}
& \hspace{.1cm} &
\includegraphics[scale=.69]{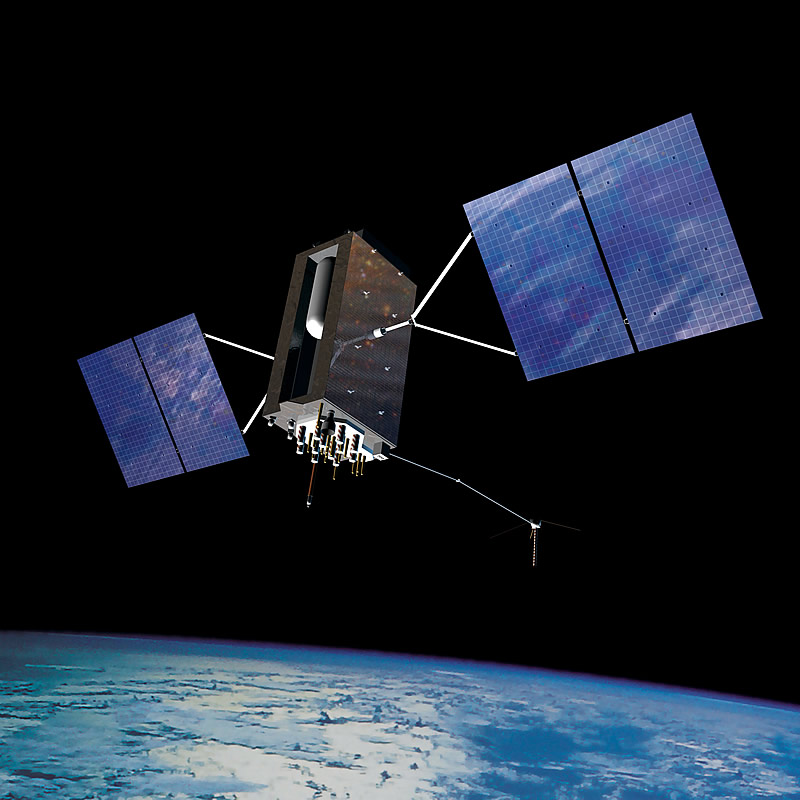}
\end{tabular}
\caption{The red vertices are a maximum independent set in the Petersen graph ($\alpha=4$). A GPS satellite: independence number calculations were used to help position the GPS III satellites.}
\end{figure}

One well-known application is the calculation of  the probability of unambiguous message transmission in a channel \cite{Lova79}. A \textit{message} consists of a string of letters. Some of these letters can be \textit{confounded} or confused; for instance ``b'' and ``d'' can be confounded. A graph can be defined consisting of the letters of the alphabet as vertices and an edge between them if they can be confounded. For a message with $k$ letters a graph can be defined with all $k$-length strings, words or messages as vertices and an edge between any pair of these strings/vertices if any pair of letters in the corresponding place between the strings can be confounded. An independent set in this graph corresponds to a  set of strings no pair of which can be confounded in any pair of places. The independence number of this graph would then represent the size of a largest dictionary of $k$-length strings  which can be sent without any risk of error. (Appropriately normalized, this number is the \textit{Shannon capacity}). This number can also be used to calculate the probability that some number of randomly chosen strings or words can be sent without error.

All existing algorithms  for finding a maximum independent set in a general graph require an exponential number of steps (in the worst case); the corresponding decision problem is NP-complete \cite{GareJohn79}.  The current boundary between possible and impossible independence number calculations in general graphs with around 2000 vertices: there is a a graph arising from error-correcting codes over an alphabet of size four, for instance, of order 2048, whose independence number has been intensively investigated by capable researchers, and is still not exactly known\footnote{\url{https://oeis.org/A265032/a265032.html}}.
Even small theoretical advances can lead to large practical payoffs.   


How can our \textsc{conjecturing} program and database of concepts, examples, theorems, and computed invariant values help? New bounds for the independence number of a graph are of both theoretical and practical interest. We can use our developed tools and resources to conjecture new bounds for the independence number of a graph, that necessarily improve on existing bounds. We can use the program to produce sequences of statements, true for all known examples in the graph theory literature, and hence unfalsifiable by any published examples. These will either admit a traditional proof or will admit counterexamples. Both theorems \textit{and} counterexamples  necessarily constitute new mathematical knowledge. 
Counterexamples, after being coded and added to our program, yield \textit{new} conjectures: because the produced conjectures must be true for all objects the program cannot re-produce a falsified conjecture. 
Newly proved bounds can be used in practical independence number calculations: in ideal cases, matching upper and lower independence number bounds can be used to exactly predict values of the independence number of a graph.



\begin{figure}
\includegraphics[width=5.2in]{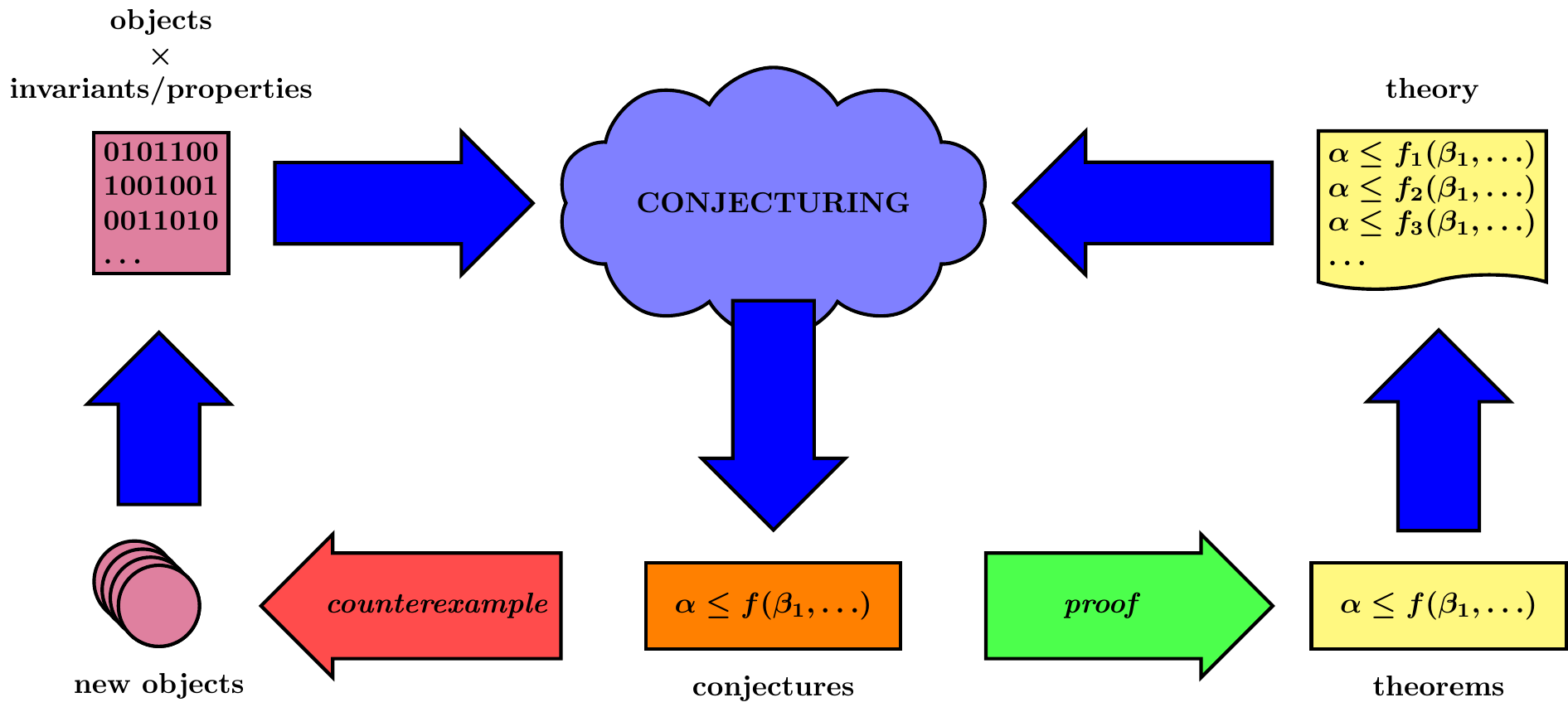}
\caption{The \textsc{conjecturing} process: (1) the program makes a conjecture, (2) if it is disproved the counterexample may be added to the program, (3) if it is true the theorem (theoretical bound) may be added to the program. In each case the process may be iterated and guaranteed to yield new conjectures.}
\end{figure}


In our Graph Brain Project summer 2017 workshop, we began with no coded theoretical knowledge---as a demonstration for the students. The program made the following not-existing-in-the-literature conjectures, which we quickly proved. 
(And these conjectures never reappeared as we began to add theorems---theoretical knowledge---to our program suggesting that these  theorems are implied by other existing theorems.)

\begin{figure}
\includegraphics[width=4.7in]{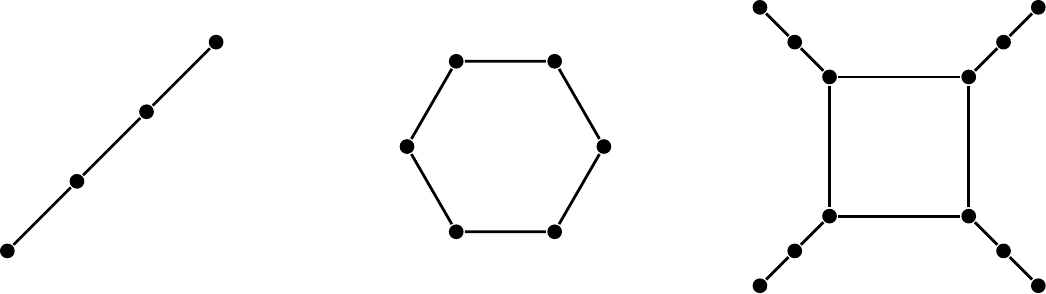}
\caption{The $r$-ciliates $C_{2,1}$,  $C_{6,0}$, and  $C_{4,2}$, with radii 2, 3, and 4, respectively.}
\end{figure}

The \textit{eccentricity} of a vertex is the maximum distance from that vertex to any other vertex in the graph.
The \textit{radius} $r$ of a graph is the minimum eccentricity of any vertex. The \textit{order} $n$ of a graph is the number of vertices of the graph. The main tool of the following proof is a theorem due to Fajtlowicz \cite{Fajt88b} that implies that every connected graph with radius $r$ has an induced subgraph of radius $r$, called an $r$-\textit{ciliate} $C_{p,q}$, consisting of a cycle with $p$ vertices with each vertex amalgamated to a path with $q$ vertices (it follows that $r=p+q$).

\begin{thm}
For any connected graph $G$, $\alpha(G) \leq n(G) -r(G)$. 
\end{thm}

\begin{proof}
Let $G$ be a connected graph with radius $r$, and $r$-ciliate $C_{p,q}$ 
(with $r=p+q$). Note that an $r$-ciliate is bipartite. It is easy to check that $n(C_{p,q})=2p(q+1)$, 
$\alpha(C_{p,q})=p(q+1)$, and 
$\alpha(C_{p,q})\leq n(C_{p,q})-r(C_{p,q})$.

Let $V'=V(G)\setminus V(C_{p,q})$, 
and $n'=|V'|$. 
Then $\alpha(G)\leq \alpha(C_{p,q})+n'\leq (n(C_{p,q})-r(C_{p,q}))+n'=(n(G)-n')-r(G)+n'=n(G)-r(G).$
\end{proof}

The \textit{degree} of a vertex in a graph is the number of vertices to which it is adjacent. The \textit{maximum degree} $\Delta$ of a graph is the largest degree of any vertex. The \textit{triangle number} $T$ of a graph is the number of triangles induced by triples of vertices of a graph. The following conjecture, weak in general, gives equality for star graphs. Our database contains only connected graphs. In this case the statement holds for any graph (connected or not) and proving the general case is easier than proving the more specialized (connected) case---an observation any mathematician will recognize. 
\begin{thm}
For any graph $G$, $\alpha(G) \geq \Delta(G) -T(G)$. 
\end{thm}

\begin{proof}
The statement can be verified for small graphs. Assume it is true for graphs with fewer than $m$ edges. 
Let $G$ be a graph with $m$ edges and $v$ be a vertex of maximum degree. If every edge is incident to $v$ then $G$ is a star, $T(G)=0$ and equality holds. It is also easy to see that the conjecture is true in any case where $T(G)=0$. So we can assume there is an edge $e$ not incident to $v$ in some triangle. Let $G'$ be the graph formed by removing edge $e$ (but not its incident vertices). So, by assumption, $\alpha(G') \geq \Delta(G') -T(G')$. We see that $\alpha(G')-1\leq \alpha(G)$, 
$\Delta(G')=\Delta(G)$ and that $T(G')+1\leq T(G)$. 
Then $\alpha(G)\geq\alpha(G')-1\geq(\Delta(G')-T(G'))-1\geq \Delta(G)-(T(G)-1)-1=\Delta(G)-T(G).$
\end{proof}

We now have available  520 graphs, 159 invariants, and 92 properties. Many of these graphs, invariants and properties were already coded into the Sage mathematical computing environment (\cite{sage17}, used for this research) by interested researchers. All of the graphs are either published graphs, or graphs which were counterexample to conjectures of our \textsc{conjecturing} program. Many of the graphs and invariants were coded during our 2017 summer research project.

An important feature of our \textsc{conjecturing}  program is the ability to use \textit{theoretical knowledge}. If $\beta$ is an invariant, proved to be an upper bound for an invariant $\alpha$, it can be a added to a \textit{theory} list; the program will then not include any expression $\gamma$ (invariant function) to its list of potentially output conjectures unless it is the case that there is a stored object $G$ such that $\gamma(G)$ is both less than the value of every previously stored conjectured bound for $G$ \textit{and} less than every stored theory bound. The stored conjecture, if true, is necessarily new knowledge---in the sense that it cannot be implied by the stored theoretical knowledge.

We have been collecting independence number bounds for graphs for some time: many are cataloged in \cite{Will11}. The ten bounds recorded here seem to be the most useful in practice. They should be interpreted for connected graphs (although most hold for general graphs). These can all be computed efficiently; thus the minimum of these upper bounds and the maximum of the lower bounds are themselves efficiently computable bounds.\\

\textbf{Six Upper Bounds for the Independence Number of a Graph}\\

(1) \texttt{independence number <= annihilation number} \cite{Pepp09}.

If the degrees, $d_1\leq d_2\leq \ldots d_n$, of the vertices of a graph are arranged in non-decreasing order,  the \textit{annihilation number} is then defined to be the largest index $k$ such that the sum of the degrees of the first $k$ vertices is no more than the sum of the degrees of the remaining vertices.

(2) \texttt{independence number <= fractional independence number} \cite{NemhTrot75}.

The independence number can be computed by finding the optimum value of an integer linear program. (For each vertex $v_i$ let $x_i\in\{0,1\}$. The objective is to maximize $\sum x_i$, where $x_i+x_j\leq 1$ for every edge $v_iv_j$.) The \textit{fractional independence number} is defined to be the optimal value of the relaxation of this linear program.

(3) \texttt{independence number <= Lov\'asz number} \cite{Lova79}.

The \textit{Lov\'asz number} ($\vartheta$) of a graph, introduced by Lov\'asz in 1979, has a large number of equivalent definitions \cite{Knut94}, one of which is the minimum of the largest eigenvalue of all the real symmetric matrices of the order of the graph with $1$s on the diagonal and $(i,j)$-entry $1$ whenever vertex $i$ is not adjacent to vertex $j$. It is an amazing fact that the Lov\'asz theta invariant can be computed efficiently \cite{GrotLovaSchr81}.

(4) \texttt{independence number <= Cvetkovi\'c bound}  \cite{CvetDoobSach95}.

The \textit{Cvetkovi\'c bound} is the minimum of the number of non-negative and non-positive eigenvalues of the adjacency matrix of the graph. 

(5) \texttt{independence number <= order - matching number}.

The \textit{matching number} is the largest number of edges none of which shares an endpoint with another.  This easy-to-prove, and sometimes useful,  bound seems to belong to the folklore of our subject.

(6) \texttt{independence number <= Hansen-Zheng bound}  \cite{HansZhen93}.

The \textit{Hansen-Zheng bound} is 
$\lfloor\frac{1}{2} + \sqrt{\frac{1}{4} + \text{order}^2 - \text{order} - 2\cdot\text{size}}\rfloor$. Here the \textit{size} is the number of edges of the graph.\\

\begin{figure}
\begin{tcolorbox}
{\small
\begin{tabular}{c|c|c}
Graph & Upper Bound & Value\\
\hline\hline
$K_5$ & \texttt{annihilation number} & 1\\
 & \texttt{fractional independence number} & 2\\
 & \texttt{Lov\'asz number} & 2.5\\
 & \texttt{Cvetkov\'ic bound} & 1\\
 & \texttt{order - matching} & 3\\
 & \texttt{Hansen-Zheng bound} & 1\\
 \hline
$C_5$ &\texttt{annihilation number} & 2\\
 & \texttt{fractional independence number} & 2.5\\
 & \texttt{Lov\'asz number} & 2.236\\
 & \texttt{Cvetkov\'ic bound} & 2\\
 & \texttt{order - matching} & 3\\
 & \texttt{Hansen-Zheng bound} & 3\\
 \hline
$K_{2,3}$ & \texttt{annihilation number} & 3\\
 & \texttt{fractional independence number} & 3\\
 & \texttt{Lov\'asz number} & 3\\
 & \texttt{Cvetkov\'ic bound} & 4\\
 & \texttt{order - matching} & 3\\
 & \texttt{Hansen-Zheng bound} & 3\\
 \hline
 Petersen & \texttt{annihilation number} & 5\\
 & \texttt{fractional independence number} & 5\\
 & \texttt{Lov\'asz number} & 4\\
 & \texttt{Cvetkov\'ic bound} & 4\\
 & \texttt{order - matching} & 5\\
 & \texttt{Hansen-Zheng bound} & 8\\

\end{tabular}
}
\end{tcolorbox}
\caption{Example upper bounds for the independence number $\alpha$ of selected graphs. $K_5$ and $C_5$ are the complete graph and cycle on five vertices; $K_{2,3}$ is the complete bipartite graph with partite sets of sizes two and three. The true values are: $\alpha(K_5)=1$, $\alpha(C_5)=2$, $\alpha(K_{2,3})=3$, $\alpha(Petersen)=4$.}
\end{figure}

\textbf{Four Lower Bounds for the Independence Number of a Graph}\\

(1) \texttt{independence number >= radius}  \cite{ErdoSaksSos86}.

The \textit{radius} was defined above. The proof that radius-critical subgraphs are $r$-ciliates immediately implies  this result as a corollary.

(2) \texttt{independence number >= residue}  \cite{FavaMaheSacl91}.

If the degrees of the vertices of a graph are arranged in non-increasing order, the Havel-Hakimi theorem says that the sequence formed by removing the first of these $d$ and reducing each of the next $d$ terms is the degree sequence of a graph. It follows that after iterating this procedure some number of times (and rearranging the new terms in non-increasing order) you get a sequence of $0$s. The number of $0$s is the \textit{residue} of the graph. 

(3) \texttt{independence number >= critical independence number}  \cite{Lars11}.

The \textit{critical independence number} is defined to be the cardinality of a certain independent set---and the theorem is trivial. It turn out that this number equals the independence number for a large class of graph (the K\"onig-Egervary graphs) which include the bipartite graphs.

(4) \texttt{independence number >= max\_even\_minus\_even\_horizontal}  \cite{Grig11}.

Let $v$ be any vertex. It is easy to show that the number of vertices at even distance from $v$ minus the number of edges induced by these vertices is a lower bound for the independence number. The \textit{max even minus even horizontal} invariant is the maximum of these values over all of the vertices of the graph is then also a lower bound for the independence number. Fajtlowicz defined this invariant and observed that it is very good in practice (at least for small graphs).

\begin{figure}
\begin{tcolorbox}
{\small
\begin{tabular}{c|c|c}
Graph & Lower Bound & Value\\
\hline\hline
$K_5$ & \texttt{radius} & 1\\
 & \texttt{residue} & 1\\
 & \texttt{critical independence number} & 0\\
 & \texttt{max\_even\_minus\_even\_horizontal} & 1\\
 \hline
$C_5$ & \texttt{radius} & 2\\
 & \texttt{residue} & 2\\
 & \texttt{critical independence number} & 0\\
 & \texttt{max\_even\_minus\_even\_horizontal} & 2\\
 \hline
 
$K_{2,3}$ & \texttt{radius} & 2\\
 & \texttt{residue} & 2\\
 & \texttt{critical independence number} & 3\\
 & \texttt{max\_even\_minus\_even\_horizontal} & 3\\
 \hline
 Petersen & \texttt{radius} & 2\\
 & \texttt{residue} & 3\\
 & \texttt{critical independence number} & 0\\
 & \texttt{max\_even\_minus\_even\_horizontal} & 1\\
\end{tabular}
}
\end{tcolorbox}
\caption{Example lower bounds for the independence number $\alpha$ of selected graphs. The true values are: $\alpha(K_5)=1$, $\alpha(C_5)=2$, $\alpha(K_{2,3})=3$, $\alpha(Petersen)=4$.}
\end{figure}

%

%
%
%


%

If the \textsc{conjecturing} program were given all published invariants in a mathematical field, all real-number operators used by mathematicians, and all published bounds, the program would necessarily produce new conjectures (not implied by existing theory) that are as simple as any human can produce (with respect to the objects that it knows). 
That is, if a human were to produce a simpler conjecture that is true for all objects the computer knows then, necessarily and by the design of the program, the conjecture must either be false for one of these objects or it must be implied, with respect to these objects, by the conjectures that the program does produce: that is, the produced conjectures must give bounds that are at least as good as the human conjecture. 
The program necessarily will consider every simpler conjecture: it iteratively generates and evaluates every single syntactically possible statement in order from the least complex to more complex. At the moment it considers the human's conjecture, if it does not produce the conjecture itself, it is because it is either false, or not \textit{significant} in the described sense.

The following two conjectures were generated in our summer workshop; they have been verified for all of the more than 14 million connected small-order graphs ($n\le 10$). In addition, we've used random graph generators to test these conjectures on a large sample of random graphs of assorted models (including Erd\H{o}s-Renyi graphs, random regular graphs, random bounded tolerance graphs, random interval graphs, and random bipartite graphs).  For each model, we tested many instances with a wide variety of parameters (also randomly generated within the given parameter space) and orders up to at least 100.

(1) \texttt{independence\_number >= min(girth, floor(lovasz\_theta))}

This is how an output conjecture of the program \textsc{conjecture} appears to a user. In particular it is an unquantified open sentence that must be interpreted. Since we used only connected graphs in this investigation, we interpret this over \textit{all} connected graphs: that is, For every connected graph $x$ \texttt{independence\_number(x) >= min(girth(x), floor(lovasz\_theta(x)))}. 

Here \textit{girth} and \textit{Lov\'asz theta} are graph invariants, while \textit{min} and \textit{floor} are real-number operators. The (lower) bound on the right-hand side of this inequality has complexity-4. The \textit{girth} of a graph is the number of edges of a smallest cycle in the graph; it can be computed efficiently. 

The Lov\'asz theta number is, in fact,  the best upper bound in practice for estimating the independence number of a graph; and, since the independence number is integral, the floor of this number must be an upper bound. It is interesting to note that here we have a conjectured lower bound for the independence number expressed in terms of the best upper bound. The conjecture can then be restated: for any connected graph, either the independence number is at least as big as its girth or the independence equals the floor of its Lov\'asz theta number.

%

(2) \texttt{independence\_number <= (average\_distance)\^{}(degree\_sum)}

We interpret this conjecture for connected graphs. The \textit{average distance} is the average of the distances between distinct pairs of vertices in the graph. This invariant is actually a lower bound for the independence number of a graph \cite{Chun88}. The \textit{degree sum} is the sum of the degrees of the vertices of the graph. Here the caret ``\texttt{\^{}}'' is the exponentiation operator \footnote{Jianxiang Chen suggests a proof sketch at: https://github.com/math1um/objects-invariants-properties/issues/359}.

\begin{figure}
\begin{tcolorbox}
\begin{lstlisting}[basicstyle=\tiny]
independence_number(x) >= minimum(girth(x), floor(lovasz_theta(x)))
independence_number(x) >= minimum(diameter(x), lovasz_theta(x))
independence_number(x) >= maximum(residue(x), 1/2*lovasz_theta(x))
independence_number(x) >= 2*floor(arccosh(lovasz_theta(x)))
independence_number(x) >= floor(arccosh(lovasz_theta(x)))^2
independence_number(x) >= ceil(lovasz_theta(x)) - radius(x)
independence_number(x) >= ceil(lovasz_theta(x)) - girth(x)
independence_number(x) >= floor(2*tan(matching_number(x)) - 2)
independence_number(x) >= floor(log(tan(order(x))^2)/log(10))
\end{lstlisting}
\end{tcolorbox}
\caption{Open conjectures for the lower bound of the independence number of a connected graph (that would fit this box using invariants already defined here). The full list of open upper and lower bound conjectures for the independence number may be found at: 
\protect\url{https://github.com/math1um/objects-invariants-properties/issues/421}
}
\end{figure}

\section{Big Mathematics}

Many disciplines make important use of labs and even larger-scale collaboration---Big Science.
Collaborative physics made an enormous splash recently with the discovery of gravitational waves by the LIGO consortium of more than 900 collaborating scientists \cite{AbotEtal16} (and a 2017 Nobel Prize in Physics), confirming a prediction of Einstein's theory of relativity, and pursued for more than 40 years.  

Mathematicians can also make advantageous use of labs and large-scale organization: examples of large-scale, organized, collaborative mathematics include the British WWII code-breaking groups at Bletchley Park, and (presumably) similar ongoing research at the National Security Agency (NSA).  The mathematics group at Bell (and later AT\&T) Labs could be harnessed to address problems as needed. Other impressive examples, albeit less tightly organized, might include the classification of finite simple groups and the Polymath Project. \textit{With continued research on automated mathematical discovery programs, and the development of code-bases, and mathematical databases, it is now possible to envision large-scale, organized, collaborative mathematics existing in the future. }

Our enormous mathematical knowledge bases---stored as research papers---are not being effectively or systematically exploited. Tens of thousands of mathematical research papers published each year---maybe hundreds of thousands. 
Only a small amount of this knowledge can be leveraged by any single researcher or group of researchers: the literature is simply too vast. 
Much of this knowledge can be usefully computerized so that intelligent computer assistants like \textsc{Conjecturing} can easily \textit{use} it. Many of these papers contain new concepts. 
Any of these could be useful---the real test is if they show up in conjectures that advance our mathematical goals.  
We should leverage this knowledge---by coding it---to more quickly advance our shared mathematical goals.


We maintain a database of values of invariants for most of our coded-stored graphs. Some of these values were calculated either with significant computer resources or using theoretical knowledge. Some of this overlaps other graph theory databases including House of Graphs \cite{BrinEtal13} and the Encyclopedia of Finite Graphs \cite{HoppPetr16}.  It would be useful---and more efficient---if researchers never had to repeat any of these computations. 
A universal graph theory database would be of real utility to researchers. 
 We imagine one day there may be something like a National Institute of Mathematics maintaining a variety of mathematical databases, and housing and organizing projects like this.
 

\begin{figure}
\begin{tabular}{ccc}
\includegraphics[width=1.6in]{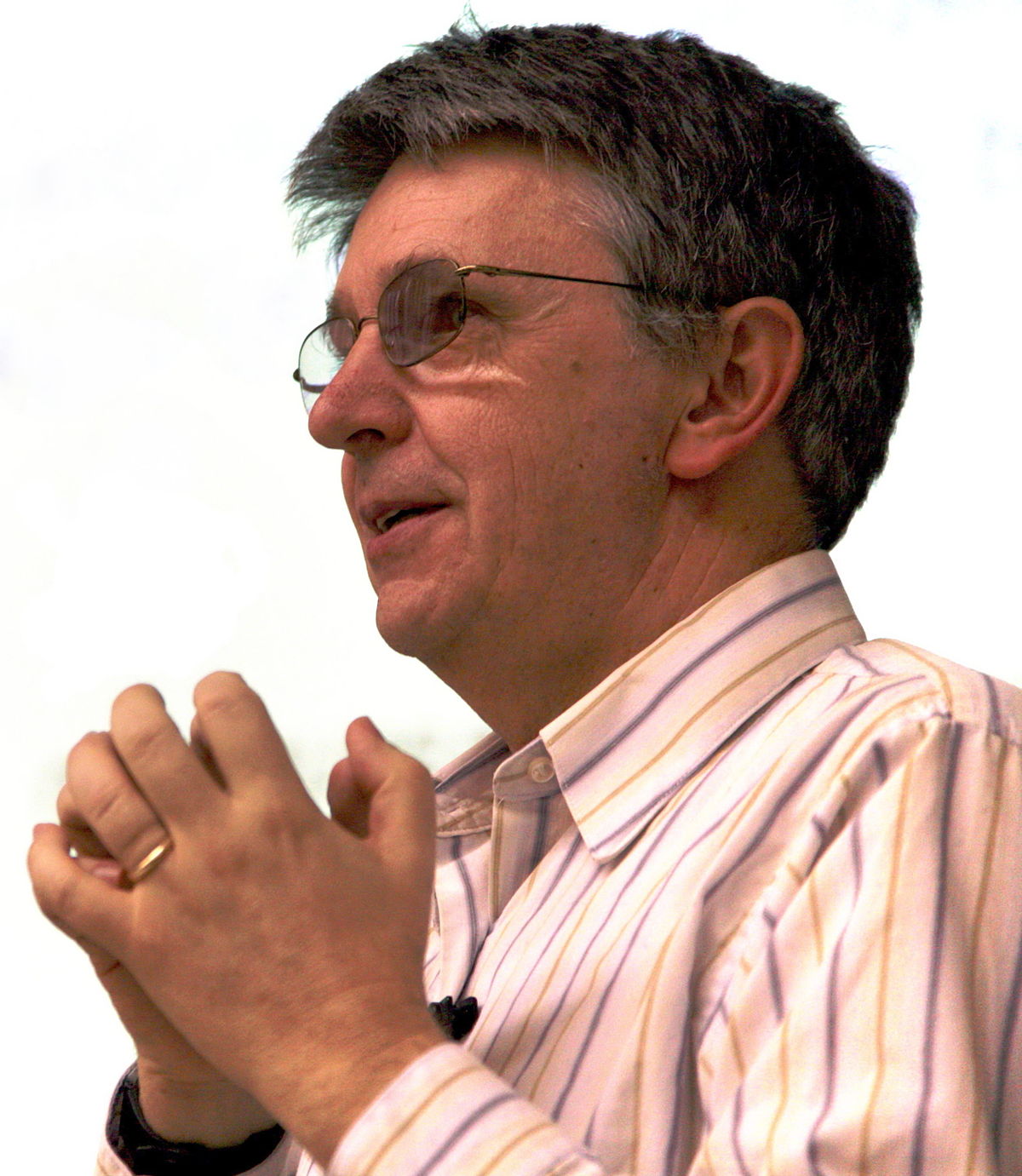}
&
\includegraphics[width=1.5 in]{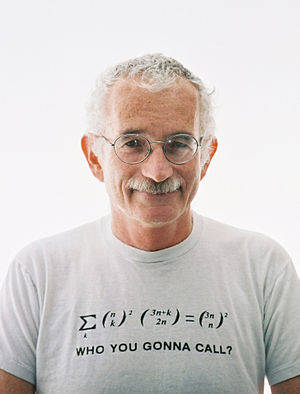}
&
\includegraphics[width=1.7in]{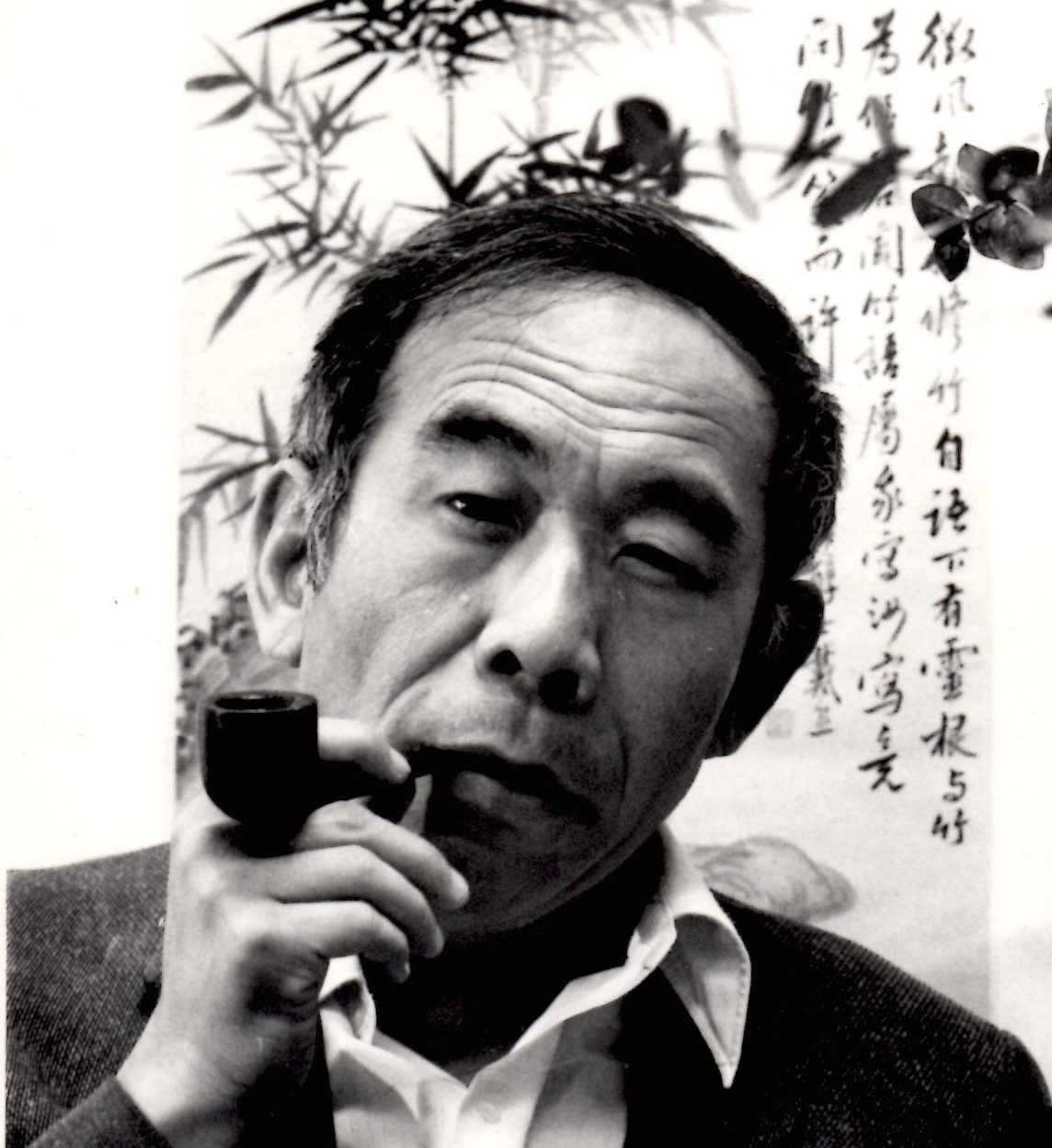}

\end{tabular}
\caption{L\'aszlo Lov\'asz, Doron Zeilberger, Hao Wang}
\end{figure}

What we have done is only a small-scale experiment, a demonstration of what is possible. It would be interesting to see the results of a \textit{large-scale} experiment. 
Continued sustained research on coding existing bounds for the independence number of a graph, generating conjectures that represent potential improvements for graphs where lower and upper bounds are not equal, proving them, adding this as a theorem, and iterating might converge on  useful and efficient independence number bounds. Even if $P\neq NP$ we might still be able to use these bounds to predict the exact value of the independence number of a graph with high probability---and this may be enough for practical purposes. Zeilberger for instance has imagined a mathematical future with results of exactly this type \cite{Zeil94}.

We have also made experiments with property-relation conjectures conjectures: these are necessary and sufficient conditions for an object to have a specified property. These conjecture types can be generated in an analogous way to the invariant-relation conjectures described so far. Examples of conjectured necessary or sufficient conditions for a graph having the property of being hamiltonian are reported in  \cite{LarsVanc17}. Much more work needs to be done here: in particular, we have coded relatively very few graph properties.


It is an important fact that successful automated discovery programs are \textit{designed} to address existing mathematical problems---and their utility is measured with respect to \textit{our} own (human) mathematical goals.
Consider for example the conjecturing program of Hao Wang. Wang was an automated mathematical discovery pioneer while he was at IBM in the late 1950s and the developer of the first conjecturing program \cite{Wang60}.  He wanted his  program to produce ``interesting'' mathematical statements---but he didn't factor in any \textit{mathematical} goal. He reported: ``The number of theorems printed out after running the machine for a few hours is so formidable that the writer has not even attempted to analyze the mass of data obtained.'' 
If some of these statements were mathematical advances Wang didn't know it. Our human goals are central to the success and (human) evaluation of our mathematical progress.

The kind of research advocated here naturally allows for the talents of researchers and students with a wide variety of abilities. Our Graph Brain Project summer 2017 workshop included students at the high school, undergraduate and graduate levels, together with faculty. The two high school students both made meaningful contributions---and learned quite a bit of graph theory along the way. They both started by coding graphs from the literature---tedious but necessary in order to achieve research literature comprehension. Both ended up doing more interesting and sophisticated coding. One of these students, with no previous coding experience, coded two different algorithms for finding the largest set of vertices in a graph that induces a bipartite subgraph. Every day  in the lab we talked about open problems at the board, discussed proof ideas, ideas for constructing counterexamples, and then chose what to work on for the day that would advance our short-term and long-term goals. 

This workshop was a natural way for researchers with a wide range of talents to work---in the same place, pushing forward research together organically, to learn, and to enjoy mathematical camaraderie. Furthermore, the natural science model of laboratories suggests ways for our students to quickly make contributions in naturally collaborative environments (which is definitely not the norm in our often isolated mathematical worlds). This might also suggest new ways to attract--and interest---a wider, more diverse,  field of mathematical talent.

Any area of mathematics where objects, invariants and properties can be coded is amenable to investigations which exactly parallel what we have described for graphs and the independence number.

\section{How to Contribute}

Two  ways to contribute to this kind of research are to either contribute to the research we have begun in graph theory, or to begin the work of coding objects, invariants, and properties for some other area of mathematics.

Our \textsc{Conjecturing} program is open-source, and written to work with Sage, an open source mathematical computational environment meant to substitute for better-known, expensive and proprietary mathematical software, and that uses Python as its interface language. This program, examples, and set-up instructions are available at: 
\url{https://nvcleemp.github.io/conjecturing/}.
Researchers in every area of mathematics can easily replicate our graph theory experiments in their own areas of research. The matrix, number and graph theory scripts we used in \cite{LarsVanc16} and other examples are also available here---these can be imitated in initial investigations. 


For graph theory we have begun to code the objects, invariants, and properties from the graph theory literature. What we've done so far includes many well-known and standard terms, familiar to all researchers. 
These are available  at:\\
\url{http://math1um.github.io/objects-invariants-properties/} \\
\noindent These are also coded for Sage. Researchers can download these, see what's been done,  and start coding---or at least add invariants to code as Github ``issues''---this a dynamic bulletin board of what needs to be coded, and what progress has been made, which any researcher anywhere can add to and comment on.  Read papers, watch talks, note new concepts and graphs and add them. It is also possible to ``fork'' what we've done. Take it, do your own thing, and build on it.
We've also pre-computed values for almost all of these invariants, for almost all of these graphs. This \textit{precomputed database} can be very useful for fast conjecturing---the program will, by default, compute any values it needs, so having pre-computed values can really speed things up.


Another way to help with our Graph Brain Project is to prove  or find counterexamples to the  open conjectures of our program: every theorem \textit{and} every counterexample count as new knowledge---and will lead to improved conjectures.
There are also many graphs for which values for certain invariants and certain properties are as-yet unknown. They need to be computed. Any new computed values can easily be added to local copies of our database and, better, if posted as a Github issue, will be included in the posted, shared copy of the database. Having these values will also improve the conjectures made by the \textsc{conjecturing} program.

In other areas of mathematics, just start! Code a few objects and invariants, and see what conjectures you get. Iterate and add. 

The computational tools we used in our graph theory investigations  included \texttt{geng} (included in the \texttt{nauty} package) for comprehensive non-isomorphic generation of all connected graphs up to any given order \cite{Mcka07}, \texttt{benzene} for the generation of benzenoid graphs \cite{BrinCapoHans02}, and \texttt{buckygen} for the generation of fullerene graphs \cite{BrinGoedMcka12}. These are very useful for searching for small counterexamples to graph conjectures. It would be useful in any other area of investigation to code similar generators for the systematic construction of example objects.

\section{Acknowledgements}

The authors are grateful for useful comments from S. Cox, R. Meagher, M. Ong Ante, J. Padden, R. Segal, N. Sloane that have greatly improved our presentation. Jianxiang Chen has been active on the Github Graph Brain Project site, finding counterexamples to conjectures



\bibliographystyle{plain}
\bibliography{../../larson.bib}
%

\end{document}